\documentclass[conference]{IEEEtran}
\IEEEoverridecommandlockouts
\usepackage{authblk}
\usepackage{cite,soul}
\usepackage{amsmath,amssymb,amsfonts,amsthm}
\usepackage{algorithm}
\usepackage{algorithmic}
\usepackage{graphicx}
\usepackage{textcomp}
\usepackage{xcolor}
\usepackage{wrapfig}
\newtheorem{theorem}{Theorem}

\newtheorem{remark}{Remark}
\usepackage{comment}
\usepackage{multirow}
\usepackage{booktabs}  
\usepackage{caption}
\usepackage{subcaption}
\usepackage{ctable}

\newcommand{\btheta}{\boldsymbol{\theta}}

\newcommand{\mcD}{\mathcal{D}}
\newcommand{\bx}{\boldsymbol{x}}
\newcommand{\by}{\boldsymbol{y}}
\newcommand{\bz}{\boldsymbol{z}}
\newcommand{\bp}{\boldsymbol{p}}

\def\BibTeX{{\rm B\kern-.05em{\sc i\kern-.025em b}\kern-.08em
    T\kern-.1667em\lower.7ex\hbox{E}\kern-.125emX}}

\newif\ifcomment
\commentfalse 
\newcommand{\ch}[1]{\ifcomment \textcolor{magenta}{Cami: #1} \fi}

\newcommand{\cra}[1]{\ifcomment \textcolor{blue}{Charul: #1} \fi}

\newcommand{\chl}[1]{\textcolor{black}{#1}}

\begin{document}

\title{Perfectly-Private Analog Secure Aggregation in Federated Learning} 

%

 \author[1]{Delio Jaramillo-Velez}
 \author[2]{Charul Rajput}
 \author[2]{Ragnar Freij-Hollanti}
 \author[2]{Camilla Hollanti\thanks{\\ \textbf{E-mails}: delio@chalmers.se, charul.rajput@aalto.fi, camilla.hollanti@aalto.fi,\\ ragnar.freij@aalto.fi, alexandre.graell@chalmers.se\\
 
   This work was supported by a joint project grant to Aalto University and Chalmers University of Technology (PIs~A.~Graell~i~Amat and C.~Hollanti) from the Wallenberg AI, Autonomous Systems and Software Program.}}
 \author[1]{Alexandre Graell i Amat}

\affil[1]{Department of Electrical Engineering, Chalmers University of Technology, Sweden}
 \affil[2]{Department of Mathematics and Systems Analysis, Aalto University, Finland }


\maketitle

\begin{abstract}

In federated learning, multiple parties train models locally and share their parameters with a central server, which aggregates them to update a global model. To address the risk of exposing sensitive data through local models, \emph{secure aggregation} \chl{via secure multiparty computation} has been proposed to enhance privacy. At the same time, \emph{perfect privacy} can only be achieved by a uniform distribution of the ``masked'' local models to be aggregated. This raises a problem when working with real-valued data, as there is no measure on the reals that is invariant under the masking operation, and hence information leakage is bound to occur. Shifting the data to a finite field circumvents this problem, but as a downside runs into an inherent accuracy--complexity tradeoff issue due to fixed-point modular arithmetic as opposed to floating-point numbers that can simultaneously handle numbers of varying magnitudes. In this paper, a novel secure parameter aggregation method is proposed that employs the torus rather than a finite field. This approach guarantees perfect privacy for each party’s data by utilizing the uniform distribution on the torus, while avoiding accuracy losses. Experimental results show that the new protocol performs similarly to the model without secure aggregation while maintaining perfect privacy. Compared to the finite field secure aggregation, the torus-based protocol can in some cases significantly outperform it in terms of model accuracy and cosine similarity, hence making it a safer choice.

\end{abstract}


\section{Introduction}

Federated learning (FL) is a decentralized machine learning framework in which multiple clients collaborate to train a global model while retaining their local data. Each participant computes local updates (e.g., gradients) based on their private datasets and shares these updates with a central server \cite{kairouz2021advances,fed_opt}. The server aggregates the received gradients and uses the resulting average to update the global model. This collaborative approach allows machine learning without the need to centralize sensitive data, making FL especially attractive in fields such as healthcare, finance, and mobile systems, where data privacy is a significant concern. However, a key challenge in federated learning is ensuring that sensitive information is not inadvertently leaked through the shared gradients. Recent research has highlighted the risks of gradient leakage, where information about the underlying data can be inferred from the updates sent to the server \cite{HAP2017, MSDS2019,yin2021comprehensive}. 
To address this, secure aggregation protocols are used to compute the average of parameter updates \cite{CC2009, BIK2017,SKRA2023}.  Secure Aggregation~\cite{zhou2022survey} is a frequently used privacy-preserving mechanism that hides individual model updates through cryptographic protocols. In essence, using lightweight secure multiparty computation~\cite{cramer2015secure} and secret sharing  ~\cite{beimel2011secret} schemes it masks the updates so that the randomness cancels out during aggregation, leaving the final model unaffected. This prevents a malicious server or a third party from extracting sensitive information about individual data points \cite{so2021securing}. However, secure aggregation protocols are typically defined over a finite field, which involves quantizing the data and mapping it to elements of a large prime-sized finite field \cite{BIK2017,Soleymani21}. This implies significant limitation when an overflow occurs, specifically when a very small number fails to have a precise representation within the chosen finite field \cite{Soleymani20}. The overflow issue compromises the correct recovery of the computation outcome, as the computation at each user is constrained by fixed-point conditions designed to prevent overflow. 

In contrast, our protocol adopts a different approach by utilizing the torus for \chl{one-time pad} encryption, which constitutes a continuous set~\cite{homomorphy_torus}. This method ensures perfect privacy, similar to that provided by finite fields, due to the existence of a uniform distribution on the torus. Moreover, this approach effectively mitigates the overflow problem, thus improving the reliability of the computation outcomes. Namely,  predicting the magnitude of numbers encountered during a learning process is challenging, making the selection of  the finite field size subtle and increasing the size directly affects the computational complexity. In contrast, adjusting the scaling parameter (defined in Section \ref{sec:about_L}) in our protocol to accommodate changes in the parameter range does not add extra computational cost. Thus, managing overflow issues is much simpler with the torus-based protocol making it a safer choice.



In the proposed secure aggregation protocol, the presence of all participants is required to fully recover the aggregated model. This ensures that no subset of colluding participants can reconstruct the model, preserving the privacy of individual updates. For this reason, we focus on the cross-silo FL setting, where all parties are assumed to participate in every training round. Privacy preserving methods in cross-silo FL have been explored from various perspectives~\cite{fl_priv_one,fl_priv_two}, including secure aggregation methods such as using double-masking techniques to completely hide the global model from the server~\cite{DHSA}, or employing homomorphic encryption-based methods~\cite{SVFL}. However, these approaches typically do not utilize masking based on a uniform distribution, which means they may fall short of offering information-theoretic security to the participating parties~\cite{Soleymani21}.
In cross-silo FL, where participants are often sensitive institutions such as hospitals \cite{Fl_medicine,MELLODY} or banks \cite{kairouz2021advances}, the privacy requirements are significantly more stringent. In these settings, perfectly secure protocols are essential, as any data leakage during training is unacceptable. Even minimal exposure can pose serious risks, potentially compromising highly sensitive information. Moreover, the impact of such leakage is difficult to quantify, as it depends on factors like the type and volume of data exposed and the potential for misuse. Therefore, ensuring perfect privacy guarantees throughout the training process is not just beneficial\textemdash it is critical.

Our main contributions are summarized as follows.
\begin{itemize}

     \item Using the uniform distribution on the torus, we present a secure aggregation protocol for Federated Averaging (FedAvg) \cite{MMRHY2017}  in the federated learning process. The torus method is not only applicable to federated learning, but to any scenario involving secure multiparty computation/secret sharing in a similar manner. 
    
     \item 
    We theoretically demonstrate that our secure aggregation protocol provides information-theoretic security for clients' data against the server.
    \item We empirically evaluate our protocol in a cross-silo FL scenario and show that it closely approximates non-secure aggregation and outperforms the secure aggregation method using finite fields. 
\end{itemize}




\section{Preliminaries}\label{prel}



\subsection{The uniform distribution on the torus}

Let $\mathbb{T} = \mathbb{R} / \mathbb{Z}$ denote the torus, which can be understood as the set of real numbers modulo 1, or equivalently, the interval $[0,1)$ with wrap-around addition. That is, two real numbers are identified if they differ by an integer. The torus $\mathbb{T}$ carries a natural structure of a $\mathbb{Z}$-module, meaning that it supports addition of elements and scalar multiplication by integers, but not multiplication between arbitrary elements. For instance,
$$1.5 \cdot 0.3 \mod 1 \neq (1.5 \mod 1) \cdot (0.3 \mod 1),$$
highlighting the absence of a well-defined multiplicative structure on $\mathbb{T}$~\cite{homomorphy_torus}.

A canonical probability distribution\textemdash i.e.,  a distribution that is invariant under the natural symmetries of the torus $\mathbb{T}$\textemdash is the uniform distribution $\mathcal{U}$ over $[0,1)$. Via the natural isomorphism between $\mathbb{T}$ and the unit circle in the complex plane (i.e., the set of complex numbers of norm 1), this uniform distribution corresponds to the Haar measure $\mu$ on $\mathbb{T}$. The Haar measure is the unique probability measure that is invariant under rotations \cite[Chapter 9]{Donald13}. Specifically, for any measurable subset $\mathcal{A} \subseteq \mathbb{T}$ and any rotation by an angle $\varphi$, the measure remains unchanged:
$$\mu(\mathcal{A}) = \mu(e^{i\varphi} \mathcal{A})\,,$$
where $e^{i\varphi} \mathcal{A} = \{e^{i\varphi} z \:|\: z \in \mathcal{A}\}$ denotes the rotated set on the complex unit circle.

\subsection{Federated Learning}

We consider a scenario with $K$ clients that jointly train a global model $\btheta\in\mathbb{R}^{m}$ under the orchestration of a central server.
Each client $k\in[K]$ owns a real value dataset $\mcD_k=\{(\bx_j^{(k)},\by_j^{(k)})|j\in[n_k]\}$ consisting of $n_k$ samples. The goal is to minimize the aggregated loss function across the clients, 
$$
\btheta^*=\arg\min_{\btheta}f(\btheta),\quad 
f(\btheta)=\sum_{k=1}^{K}\frac{n_k}{n}f_{k}(\btheta)\,,
$$
where $f_k$ is the local loss function of client $k$ computed over the local dataset $\mcD_k$. The model $\btheta$ is trained over multiple epochs. At each epoch, the clients compute the local gradient $\nabla f_k(\btheta)$ and send them to the central server. The server updates the global model through gradient descent as, 
$$
\btheta\leftarrow \btheta-\lambda\nabla_{\btheta}f(\btheta),\quad \nabla_{\btheta}f(\btheta)=\frac{1}{n}\sum_{k=1}^{K}\nabla_{\btheta}f_k(\btheta)\,,
$$
where $\lambda$ is the learning rate and $n=\sum_{k=1}^{K}n_k$.

\section{Perfectly-Private Analog Secure \\ Aggregation Protocol}\label{proto}

In this section, we present an analog secure aggregation protocol that guarantees perfect privacy in the information-theoretic sense. The key idea is to transform the model parameters computed by each party onto the torus and perform secure aggregation directly in this continuous domain. Our protocol follows the same spirit as the original secure aggregation scheme of~\cite{BIK2017}, which operates over a finite field. However, unlike their discrete encryption, we leverage a continuous representation to encrypt the model weights.
\vspace{1ex}


\noindent\textbf{Protocol:}  We assume that all clients complete the protocol and agree on a matched set of input perturbations as follows. Each client $k \in [K]$ obtains a vector $\bz_{k,j}\in \mathbb{T}^m$ by uniformly sampling the entries from $\mathbb{T}$ for every other client $j \in [K]$ such that $k < j$, and uses $\bz_{j,k}$ for each $j$ such that $k > j$. The clients $k$ and $j$ then exchange the corresponding vectors  $\bz_{k,j}$ and $\bz_{j,k}$ over their communication channel. 
\chl{Note that we do not  have to assume that all  communication links are secure, since an eavesdropper can only gain information on the local model $k$ if it can compromise all  other $K-1$ links.}
\vspace{1ex}



\noindent\textbf{Clients:} Let \chl{$\btheta_k\in \mathbb{R}^m$} be the local model of client $k$. The model parameters are mapped to the torus by 
\begin{equation}\label{UC}
    \btheta_k \mapsto \frac{1}{L}\btheta_k \mod 1,
    \end{equation}
where $L$ is a scaling factor introduced to ensure correct decryption, as explained in Section~\ref{sec:about_L},
and the operation is applied pointwise across all coordinates of the vector. Then, client $k$ encrypts its scaled local model parameters as
\begin{equation*}
\frac{1}{L}\btheta_k \mapsto\,   \bp_k = \frac{1}{L}\btheta_k+
\sum_{j=k+1}^K \bz_{k,j}-\sum_{j=1}^{k-1} \bz_{j,k}\mod 1\,. 
\end{equation*}
\vspace{1ex}

\noindent\textbf{Server:}  The central server receives the encrypted local updates $\bp_1, \bp_2, \ldots, \bp_K$ from each client. It then  \emph{aggregates} these updates as
\begin{align*}
\bz&=\sum_{k=1}^{K}\bp_k =\sum_{k=1}^{K}\left(\frac{1}{L}\btheta_k+
\sum_{j=k+1}^K \bz_{k,j}-\sum_{j=1}^{k-1} \bz_{j,k}\right)\\
&=\sum_{k=1}^{K}\frac{1}{L}\btheta_k \mod 1 
\label{eq:agg}
    \end{align*}
to obtain the scaled global model $\frac{1}{L}\btheta$ on the torus, which is sent back to the clients for the next training round. 
Since  the scaling factor $L$ is known to all clients, they can recover the global model $\btheta$ and proceed with the next  training round.


\begin{remark}
    When implementing our protocol on a computer, floating point arithmetic must be used instead of actual real numbers. Doing so, some truncation of real numbers is used, whereby the uniform measure on the torus will be replaced by a discrete uniform measure on some finite set. This is in perfect analogy with what happens in~\cite{BIK2017}, with the difference that our modular arithmetic is done on floating points in an interval, rather than on a prime field as in ~\cite{BIK2017}. This means that, if the chosen precision level (field size) is too small, computation overflow under multiplication leads to genuine information loss in the finite field, whereas only precision is lost in the torus case.
\end{remark}

\subsection{Choice of scaling factor}
\label{sec:about_L}

Given that the function mapping real numbers to the torus, as defined in \eqref{UC}, is not one-to-one, it is essential to carefully manage this mapping to minimize accuracy loss. To address this, we introduce a scaling factor $L$.  This factor ensures that the numbers resulting from the encryption and decryption operations remain in the interval $[0,1)$.
The appropriate choice of $L$ depends on two key factors:
\begin{enumerate}
    \item The range of local model updates, $R\in\mathbb{R}$,
    such that $||\btheta_k||_{\infty}\leq R$ for all $k\in [K]$, where $||\cdot||_{\infty}$ is the maximum norm.
    \item The number of participating clients, $K$.
\end{enumerate}

To address the first point, if $||\btheta||_{\infty}=R$, then to keep the numbers between $[0, 1)$ during encryption, we need to scale by the factor $\frac{1}{R}$.    
To address the second point, we need to scale by the factor $\frac{1}{K}$ to keep the numbers between $[0, 1)$ during decryption.
Therefore, we can choose the value of the scaling factor as 
$$L\geq KR.$$ 

\begin{remark}
When the model parameters are small, for which $R\leq 1$, secure aggregation over a finite field can suffer from overflow issues if the filed size is not chosen large enough (see Figures~\ref{fig:one}, and ~\ref{fig:two}). In the case of secure aggregation on the torus, the scaling factor $L$ can simply be set to be greater than or equal to the number of clients $K$ to completely recover the model (explained in detail in Section \ref{about_L2}). On the other hand, if overflow occurs in a finite field due to very large parameter values, the proposed protocol can still be used successfully after carefully determining an appropriate $L$. In this case, an approximate range of parameters is required to select $L$, which can be subsequently adjusted during the learning process, as necessary. 
\end{remark}



\begin{remark} Unlike in the finite field case, adjusting the value of the scaling parameter $L$ according to changes in the parameter range does not add any extra computational load. However, increasing the field size in the finite field case significantly increases the complexity of all computations. In other words, handling the problem of overflow is considerably simpler compared to the finite field case.
\end{remark}


\subsection{Security analysis}

\begin{theorem}\label{privacy}
   Consider a federated learning training with $K$ clients. Then, our aggregation protocol guarantees that the updates of parameters at clients $1,2, \ldots, K$, denoted as $\btheta_1, \btheta_2, \ldots, \btheta_K$, respectively, remain information-theoretically private, i.e., for a client $k\in [K]$, $$I(\btheta_k; (\bp_{t})_{t\in [K]})=0$$ 
   where $\bp_{t}$ denotes the encrypted update of the parameters from party $t\in [K]$, and $I$  the mutual information.

\end{theorem}

\begin{proof}

In our protocol, the entries of $ \bz_{t,j}$ are drawn from the uniform distribution defined by the Haar measure on the torus, independently of $(\btheta_k)_{k\in [K]}$.  The encrypted value $\bp_t$ of the local model $\btheta_t$ at client $t$ is defined as
$$
\bp_t=\frac{1}{L}\btheta_t+
\sum_{j=t+1}^K \bz_{t,j}-\sum_{j=1}^{t-1} \bz_{j,t}\mod 1.
$$
Thus, we conclude that the entries in \( \bp_t \) are uniformly distributed over \([0, 1)\) both before and after conditioning on~$(\btheta_k)_{k\in[K]}$. 
Therefore, for any given $k\in [K]$,
\begin{align*}
        I(\btheta_k; (\bp_t)_{t\in[K]}) 
        &=H((\bp_t)_{t\in[K]})-H((\bp_t)_{t\in[K]}|\btheta_{k})=0,
    \end{align*}
where $H$ denotes the entropy.
\end{proof}

\subsection{Complexity analysis}

In Table~\ref{tb:1}, we summarize the computational, communication, and storage complexities for the clients and the server assuming $m$ model parameters. The complexities are expressed in terms of the number of real operations/symbols.
\begin{table}[t!]
 \caption{Complexity summary
 \label{tb:1}}
 \vspace{-3ex}
 \begin{center}
 \scalebox{.8}{ 
 \begin{tabular}{cc c } 
  \toprule
 Cost &  For clients &  For server \\
 \midrule
 Computational & $m^2+m(K-1)$ & $m(K-1)$  \\ [1ex]
 Communication & $mK$  & $mK$ \\[1ex]
 Storage & $mK$ & $m$ \\ [1ex] 
  \bottomrule
 \end{tabular}
 }
 \end{center}
 \vspace{-6ex}
 \end{table}

For the clients, the computational complexity involves mapping the updated model parameters to the torus, which requires $m$ multiplications and encrypting the parameters, which requires $K-1$ additions/subtractions per parameter. This results in a total complexity of $m^2+m(K-1)$. The communication complexity includes sharing random values with other clients for each model parameter, which amounts to $(K-1)m$ \chl{real symbol} transmissions and sending $m$ encrypted parameters to the server. The total communication complexity is thus $(K-1)m + m = mK$. The clients need to store $(K-1)m$ random points on the torus and $m$ parameter updates, yielding a total storage complexity of $mK$.


For the server, the computational complexity of decrypting all parameter updates is $m(K-1)$. The communication complexity for sending data to the clients is $mK$, and the storage complexity involves the update of $m$ parameters.


 

  



  

\section{Experiments}\label{exper}

We consider a standard cross-silo federated learning scenario, using  federated averaging as the aggregation method. We conduct experiments with 5, 10, 15, 20, and 30 clients \cite{realistic_health,MELLODY}, on the MNIST and CIFAR-10 image datasets, utilizing a single-layer neural network for MNIST and a ResNet-18 model pretrained on ImageNet for CIFAR-10. Both datasets (containing 60,000 and 50,000 samples, respectively) are randomly and evenly distributed among the clients, ensuring homogeneous data distribution for training. We use the cross-entropy loss and stochastic gradient descent (SGD). For MNIST, we use a learning rate of 0.01, a batch size of 64, and a single local epoch per round. For CIFAR-10, we incorporate momentum and consider the hyperparameters from \cite{BIK2017}: a learning rate of 0.05, momentum of 0.9, and a weight decay of 0.001. Furthermore, the batch size is set to 128, and the number of local epochs is set to 5. The reported results are averaged over 10 runs for MNIST and 5 runs for CIFAR-10. 

We use PyTorch for our simulations, which employs a default 32-bit floating-point representation for the model parameters. This precision is maintained for secure aggregation over the torus, and we fix $L$ equal to the number of clients. 
For secure aggregation over a finite field, we examine two scenarios: one where the finite field size matches the floating-point representation of the model parameters and another where a mismatch occurs. We evaluate the performance  using two metrics: (i) the top accuracy achieved by the models, and (ii) the cosine similarity between the global model obtained with secure aggregation and no secure aggregation. The cosine similarity measures the angle between two model vectors; a value of $1$ indicates that the models point in exactly the same direction.

\begin{figure*}[!t]
    \centering
    \includegraphics[width=0.8\textwidth]{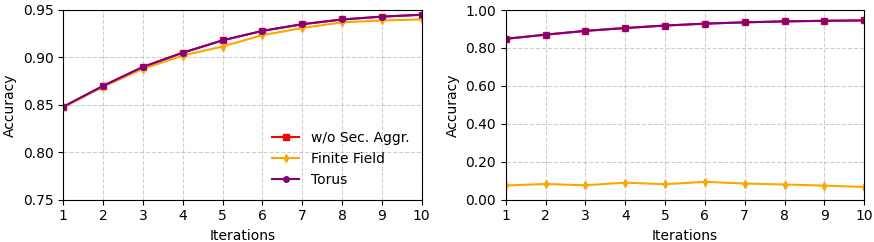}
    \caption{Accuracy of the three aggregation protocols on the MNIST dataset with $K=10$ clients and $L=K$. The test accuracy of the global model is shown using finite field secure aggregation with  field sizes of $2^{31}-1$ (left) and  $2^{15}-1$ (right).}
    \label{fig:one}
    \vspace{-1ex}
\end{figure*}

\begin{figure*}[!t]
    \centering
    \includegraphics[width=0.8\textwidth]{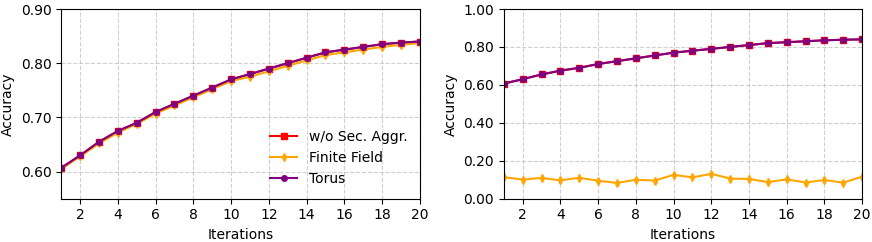}
    \caption{Accuracy of the three aggregation protocols on the CIFAR-10 dataset with $K=10$ clients and $L=K$. The test accuracy of the global model is shown using finite field secure aggregation with  field sizes of $2^{31}-1$ (left) and  $2^{15}-1$ (right).}
    \label{fig:two}
\end{figure*}

\subsection{Secure aggregation on finite fields with size $2^{31}-1$}
In this case, the prime cardinality for the finite field is set to $p = 2^{31} - 1$. We use a 7-degree precision for fixed-point operations, as referenced in \cite{precision_values}. We analyze the impact of increasing the number of clients on the model accuracy and the cosine similarity between the model without secure aggregation and those with secure aggregation.

As shown in Table~\ref{cosine_32} and Table~\ref{accuracy_32}, for both datasets, the impact of secure aggregation on the torus is negligible, as we can recover the exact model (see also Figures~\ref{fig:one}, and ~\ref{fig:two}). However, in the case of (modular) fixed-point aggregation, the impact on the accuracy and cosine similarity, though initially minimal, becomes more pronounced as the number of clients increases.


\begin{table}[!t]
\caption{Cosine similarity between the global model and the securely aggregated global models for the MNIST and CIFAR-10 datasets. Finite field size $2^{31}-1$, $L=K$.}
\vspace{-3ex}
\begin{center}
\scalebox{.8}{ 
\begin{tabular}{ccc}
\toprule
Clients $K$ &Finite Fields & Torus  \\
\midrule
\multicolumn{3}{c}{MNIST} \\
\midrule
 $5$ & $0.994 \pm .001$ & $1.000 \pm .000$  \\
$10$ & $0.988 \pm .006$ & $1.000 \pm .000$  \\
$15$ & $0.978 \pm .016$ & $1.000 \pm .000$   \\
$20$ & $0.964 \pm .025$ & $1.000 \pm .000$   \\
$30$ & $0.956 \pm .008$ & $1.000 \pm .000$   \\
\midrule
\multicolumn{3}{c}{CIFAR-10} \\
\midrule
$5$ & $0.972 \pm .004$ & $1.000 \pm .000$  \\
$10$ & $0.945 \pm .012$ & $1.000 \pm .000$  \\
$15$ & $0.914 \pm .023$ & $1.000 \pm .000$   \\
$20$ & $0.897 \pm .013$ & $1.000 \pm .000$   \\
$30$ & $0.871 \pm .034$ & $1.000 \pm .000$   \\
\bottomrule
\end{tabular}
\label{cosine_32}
}
\end{center}
\vspace{-2ex}
\end{table}

\begin{table}[!t]
\caption{Accuracy of the global models for the MNIST and CIFAR10 datasets. Finite field size $2^{31}-1$, $L=K$.}
\vspace{-3ex}
\begin{center}
\scalebox{0.8}{ 
\begin{tabular}{cccc}
\toprule
Clients $K$ & w/o Secure Aggr. & Finite Fields& Torus\\
\midrule
\multicolumn{4}{c}{MNIST} \\
\midrule
$5$ & $0.959 \pm .080$&$0.946 \pm .030$ & $0.959 \pm .080$  \\
$10$ &$0.943 \pm .099$&$0.940 \pm .098$ & $0.943 \pm .099$  \\
$15$ &$0.898 \pm .152$&$0.888 \pm .064$ & $0.898 \pm .152$   \\
$20$ &$0.860 \pm .095$&$0.842 \pm .081$ & $0.860 \pm .095$  \\
$30$ &$0.833 \pm .049$&$0.819 \pm .048$ & $0.833 \pm .049$  \\
\midrule
\multicolumn{4}{c}{CIFAR-10} \\
\midrule
$5$ & $0.865 \pm .035$&$0.857 \pm .031$ & $0.865 \pm .035$  \\
$10$ &$0.839 \pm .016$&$0.836 \pm .013$ & $0.839 \pm .016$  \\
$15$ &$0.822 \pm .048$&$0.809 \pm .040$ & $0.822 \pm .048$   \\
$20$ &$0.784 \pm .077$&$0.760 \pm .042$ & $0.784 \pm .077$   \\
$30$ &$0.710 \pm .051$&$0.701 \pm .014$ & $0.710 \pm .051$   \\
\bottomrule
\end{tabular}
\label{accuracy_32}
}
\end{center}
\vspace{-2ex}
\end{table}


\subsection{Secure aggregation on finite fields with size $2^{15}-1$}

To observe the overflow problem in standard secure aggregation over finite fields, in Table~\ref{cosine_16} and Table~\ref{accuracy_16} we show results for a prime cardinality for the finite field is set to $p = 2^{15} - 1$, with a precision of 4 degrees, as described in \cite{precision_values}. 
 Notably, since the torus-based protocol does not rely on a finite field, it successfully recovers the exact model. In contrast, secure aggregation over a finite field entails a significant deterioration in both the cosine similarity and accuracy, highlighting the challenges posed by limited representation capabilities in the finite field. On the left-hand side of Figures~\ref{fig:one} and~\ref{fig:two}, we observe that the accuracy of the models is not affected by the secure aggregation protocol, as the finite field size is sufficiently large. However, on the right-hand side of the same figures, we can see a significant drop in the model performance when a finite field with a smaller size is used.


\begin{table}[!t]
\caption{Cosine similarity between the global model and the securely aggregated global models for the MNIST and CIFAR-10 datasets, with the finite field size $2^{15}-1$, and $L$ equal to the number of clients.}
\vspace{-3ex}
\begin{center}
\scalebox{0.8}{ 
\begin{tabular}{ccc}
\toprule
Clients&Finite Fields& Torus  \\
\midrule
\multicolumn{3}{c}{MNIST} \\
\midrule
$5$ & $0.675 \pm .013$ & $1.000 \pm .000$  \\
$10$ & $0.440 \pm .093$ & $1.000 \pm .000$  \\
$15$ & $0.327\pm .161$ & $1.000\pm .000$   \\
$20$ & $0.324 \pm .178$ & $1.000 \pm .000$   \\
$30$ & $0.104 \pm .060$ & $1.000 \pm .000$   \\
\midrule
\multicolumn{3}{c}{CIFAR-10} \\
\midrule
$5$ & $0.871 \pm .001$ & $1.000 \pm .000$  \\
$10$ & $0.716 \pm .041$ & $1.000 \pm .000$  \\
$15$ & $0.672 \pm .052$ & $1.000 \pm .000$   \\
$20$ & $0.644 \pm .072$ & $1.000\pm .000$   \\
$30$ & $0.452 \pm .016$ & $1.000\pm .000$   \\
\bottomrule
\end{tabular}
\label{cosine_16}
}
\end{center}
\vspace{-5ex}
\end{table}

\begin{table}[!t]
\caption{Accuracy of the global models for the MNIST, and CIFAR-10 datasets with finite field size $2^{15}-1$, and $L$ equal to the number of clients.}
\vspace{-3ex}
\begin{center}
\scalebox{.8}{ 
\begin{tabular}{cccc}
\toprule
Clients&w/o Secure Aggr.&Finite Fields& Torus\\
\midrule
\multicolumn{4}{c}{MNIST} \\
\midrule
$5$ & $0.959 \pm .080$&$0.137 \pm .042$ & $0.959 \pm .080$  \\
$10$ &$0.943 \pm .099$&$0.067 \pm .021$ & $0.943 \pm .099$  \\
$15$ &$0.898 \pm .152$&$0.065 \pm .029$ & $0.898 \pm .152$   \\
$20$ &$0.860 \pm .095$&$0.064 \pm .030$ & $0.860 \pm .095$  \\
$30$ &$0.833 \pm .049$&$0.044 \pm .124$ & $0.833 \pm .049$  \\

\midrule
\multicolumn{4}{c}{CIFAR-10} \\
\midrule
$5$ & $0.865 \pm .035$&$0.177 \pm .028$ & $0.865 \pm .035$  \\
$10$ &$0.839 \pm .016$&$0.117 \pm .017$ & $0.839 \pm .016$  \\
$15$ &$0.822 \pm .048$&$0.090 \pm .024$ & $0.822 \pm .048$   \\
$20$ &$0.784 \pm .077$&$0.090 \pm .026$ & $0.784 \pm .077$   \\
$30$ &$0.710 \pm .051$&$0.086 \pm .012$ & $0.710 \pm .051$   \\
\bottomrule
\end{tabular}
\label{accuracy_16}
}
\end{center}
\vspace{-4ex}
\end{table}



\subsection{Secure aggregation on the torus for different values of $L$}\label{about_L2}

As mentioned in Subsection~\ref{sec:about_L}, our protocol can also be affected by overflow problems depending on the value of $L$. However, unlike the finite field secure aggregation protocol, where the overflow is influenced by the exact values of the model parameters, $L$ is primarily determined by the number of clients. In Table~\ref{values_L}, we present the results for different values of $L$. When $L$ is set to a value smaller than the number of clients, we observe a decay in model performance, and the secure aggregated model diverges in a different direction. This issue is corrected when $L$ is set to a value greater than or equal to the number of clients.

\begin{table}[!t]
\caption{Accuracy and cosine similarity of secure aggregation on the torus for the MNIST and CIFAR-10 datasets. The experiments are conducted in a scenario with 10 clients.}
\vspace{-3ex}
\begin{center}
\scalebox{0.8}{ 
\begin{tabular}{ccc}
\toprule
$L$&Accuracy& Cosine Sim.\\
\midrule
\multicolumn{3}{c}{MNIST} \\
\midrule
$1$ & $0.114 \pm .040$ & $0.673 \pm .019$  \\
$10$ & $0.943 \pm .099$ & $1.000 \pm .000$  \\
$10^2 $ & $0.943 \pm .099$ & $1.000\pm .000$   \\
$10^3$ & $0.943 \pm .099$ & $1.000\pm .000$  \\
\midrule
\multicolumn{3}{c}{CIFAR-10} \\
\midrule
$1$ & $0.138 \pm .023$ & $0.872 \pm .005$  \\
$10$ & $0.839 \pm .016$ & $1.000 \pm .000$  \\
$10^2 $ & $0.839 \pm .016$ & $1.000\pm .000$   \\
$10^3$ & $0.839 \pm .016$ & $1.000\pm .000$  \\
\bottomrule
\end{tabular}
\label{values_L}
}
\end{center}
\vspace{-5ex}
\end{table}

\section{Conclusions and Future Work}

We proposed a secure aggregation protocol for federated learning by leveraging the uniform distribution defined on the torus.  
We demonstrated that the mutual information between the local model parameters and the models shared with the server is zero, ensuring that the proposed protocol provides perfect privacy.
To evaluate the effectiveness of our protocol, we conducted experiments on a cross-silo scenario on the MNIST and CIFAR-10 datasets, showing that our protocol may significantly outperform secure aggregation over finite fields in terms of both accuracy and cosine similarity. Hence, the torus-based protocol is a safe choice providing perfect privacy and easily avoiding overflow issues. 

As future work, we will consider the presence of stragglers in the secure aggregation protocol over the torus topology and extend the protocol to support cross-device federated learning scenarios. 
To this end, we could use the method given in \cite[Section 4.0.2]{BIK2017} and implement a threshold secret sharing scheme over reals, where each client distributes shares of their randomness to all other clients. 
However, after obtaining the necessary shares of randomness from the active clients, if the server receives a delayed response from any dropped client (who was assumed to be dropped due to a significant delay in its response), the server can recover its local parameters.
\ch{Charul, do we maintain perfect privacy? If yes, please add a short explanation here, or otherwise comment that it is lost by doing this. This could be just added as part of the future work, instead of a remark. Or if we add it as a remark earlier, then it seems funny to comment on stragglers as future work if we already explained how to do it.} \cra{If all dropped users stayed dropped then we are fine. If anyone of them replies later then it's data will be exposed. I added last line for it.}
\section*{Acknowledgment}
The authors would like to thank Okko Makkonen and David Karpuk for helpful discussions.

\clearpage
\newpage

\ch{We could and maybe should add a few more references.}
\IEEEtriggeratref{13}
\bibliographystyle{ieeetr}
\bibliography{references_updated}


\end{document}